\pdfoutput=1

\documentclass[11pt]{article}

\usepackage{naacl2021}

\usepackage{times}
\usepackage{latexsym}

\usepackage[T1]{fontenc}

\usepackage[utf8]{inputenc}

\usepackage{microtype}

\newcommand{\paran}[1]{\left( #1 \right)}

\usepackage{array}

\usepackage{algorithmic}
\newcolumntype{B}{>{\centering\arraybackslash}m{3.3cm}}
\newcolumntype{L}{>{\centering\arraybackslash}p{1.5cm}}
\usepackage{amsmath,amssymb,amsfonts,bbm,amsthm}
\usepackage[ruled,vlined,linesnumbered,noend]{algorithm2e}
\let\oldnl\nl
\newcommand{\nonl}{\renewcommand{\nl}{\let\nl\oldnl}}
\makeatother
\usepackage{xspace}
\usepackage{xcolor,colortbl}
\usepackage{graphicx}

\newtheorem{theorem}{Theorem}

\title{On a Utilitarian Approach to Privacy Preserving Text Generation}

\author{Zekun Xu \\
  Amazon \\
  Seattle, WA USA \\
  \texttt{\small zeku@amazon.com} \\\And
  Abhinav Aggarwal \\
  Amazon\\
  Seattle, WA USA \\
  \texttt{\small aggabhin@amazon.com} \\\And
  Oluwaseyi Feyisetan \\
  Amazon \\
  Seattle, WA USA\\
  \texttt{\small sey@amazon.com} \\\And
  Nathanael Teissier \\
  Amazon \\
  Arlington, VA USA\\
  \texttt{\small natteis@amazon.com}}

\begin{document}

\maketitle
\begin{abstract}
Differentially-private mechanisms for text generation typically add carefully calibrated noise to input words and use the nearest neighbor to the noised input as the output word. When the noise is small in magnitude, these mechanisms are susceptible to reconstruction of the original sensitive text. This is because the nearest neighbor to the noised input is likely to be the original input.
To mitigate this empirical privacy risk, we propose a novel class of differentially private mechanisms that parameterizes the nearest neighbor selection criterion in traditional mechanisms. Motivated by Vickrey auction, where only the second highest price is revealed and the highest price is kept private, we balance the choice between the first and the second nearest neighbors in the proposed class of mechanisms using a tuning parameter. This parameter is selected by empirically solving a constrained optimization problem for maximizing utility, while maintaining the desired privacy guarantees. 
We argue that this empirical measurement framework can be used to align different mechanisms along a common benchmark for their privacy-utility tradeoff, particularly when different distance metrics are used to calibrate the amount of noise added. Our experiments on real text classification datasets show up to 50\% improvement in utility compared to the existing state-of-the-art with the same empirical privacy guarantee.

\end{abstract}

\section{Introduction}
Over the past decade, privacy-preserving machine learning has emerged as a hot topic in a variety of real world speech and language applications. In natural language processing (NLP), ensuring data privacy in machine learning tasks is especially challenging because text data tends to be rich in sensitive and potentially identifiable information about the users that contributed to these datasets.  

The literature is replete with approaches proposed for privacy-preserving text analysis, such as replacing sensitive information with general terms \cite{cumby2011machine,anandan2012t,sanchez2016c}, injecting additional words into original texts \cite{domingo2009h,pang2010embellishing,sanchez2013knowledge}, as well as k-anonymity and its variants \cite{sweeney2002k,machanavajjhala2007diversity,li2007t}. However, these methods are provably non-private and have been shown to be vulnerable to re-identification attacks \cite{korolova2009releasing,petit2015peas}. To ensure a quantifiable privacy guarantee, differential privacy (DP) has become the \emph{de facto} standard for privacy-preserving statistical analysis \cite{dwork2006calibrating,dwork2008differential,dwork2014algorithmic}, with applications to text analysis. 

At a high level, a randomized algorithm is differentially private if the output distributions from any two neighboring databases are (near) indistinguishable. This indistinguishability is controlled by a privacy parameter, which, in the case of text analysis, is often scaled by the distance between neighboring datasets to capture the semantic similarity between different words~\cite{feyisetan2019leveraging,fernandes2019generalised,feyisetan2020privacy,xu2020mahalanobis}. This calibration enables the mechanisms to enjoy \emph{metric-DP}~\cite{andres2013geo,chatzikokolakis2013broadening}, which was first introduced as a generalization of local DP~\cite{kasiviswanathan2011can} for protecting location privacy. Observe that a direct application of local DP mechanisms will be too restrictive because it requires that the probability ratio between the output distributions of any two words in the vocabulary be bounded by some fixed constant. Due to the high dimensional nature of textual tasks and very large vocabulary sizes (e.g. 2.2M words for \textsc{GloVe} common crawl~\cite{pennington2014glove}), this can lead to adding a lot of noise for achieving the desired privacy guarantees, severely impacting the utility of the NLP task. 

\noindent\textbf{Comparing Metric-DP Mechanisms.} In the context of text analysis, we are given a vocabulary set $\mathcal{W}$ and an embedding function $\phi:\mathcal{W}\to\mathbb{R}^p$, where $p$ is the dimensionality of the embedding model. For any $\epsilon>0$, a mechanism $M:\mathcal{W}\to\mathcal{W}$ is said to be $\epsilon$ differentially private with respect to a given metric $d:\mathbb{R}^p\times\mathbb{R}^p\to [0,\infty)$ if for any $w,w',\hat{w}\in\mathcal{W}$, the following holds:
\begin{equation}\label{mdp}
\frac{\Pr\{M(w)=\hat{w}\}}{\Pr\{M(w')=\hat{w}\}}\leq e^{\epsilon d\{\phi(w),\phi(w')\}}.
\end{equation} 
The probabilistic guarantee in (\ref{mdp}) ensures that the log probability ratio of observing any output $\hat{w}$ given two inputs $w$ and $w'$ is bounded by $\epsilon d\{\phi(w),\phi(w')\}$. This makes metric-DP less restrictive in that the indistinguishability of the output distributions is scaled by the distance between the inputs. If $d\{\phi(w),\phi(w')\}=\mathbbm{1}(w\ne w')$, then metric-DP reduces to standard DP.

Note that while metric-DP allows for a flexible privacy budget calibrated by not only $\epsilon$ but also the distance metric, this flexibility makes it harder to interpret the privacy parameter $\epsilon$. For example, in standard DP, $\epsilon=30$ essentially means negligible privacy guarantee since $e^{30}$ is an astronomically large probability ratio; however, $\epsilon=30$ is common in the metric-DP literature \cite{fernandes2019generalised,feyisetan2020privacy,xu2020mahalanobis} and still provides meaningful privacy guarantees. This is because the pairwise distance in the word embedding space can be small floating numbers, which brings $\exp\paran{30d\{\phi(w),\phi(w')\}}$ to a reasonable scale. Thus, $\epsilon$ alone cannot fully characterize the privacy guarantee without the knowledge of the underlying metric space. More importantly, this indicates that the privacy guarantees from DP mechanisms with respect to different metrics are not directly comparable using only their $\epsilon$ values.
 
\paragraph{Our Contributions.} A common feature in the existing metric-DP text generation mechanisms is to add a calibrate noise to the input word embedding and then output the nearest neighbor to the noisy embedding as the output.
 However, when the additive noise is small in magnitude, the input word is likely to remain unchanged, which may constitute an empirical privacy risk because it is trivial for the adversary to reconstruct the original word. To mitigate this issue, we present a novel class of metric-DP text generation mechanisms in this paper. Motivated by the Vickrey auction \cite{vickrey1961counterspeculation} scheme, also known as the second-price auction, 
we refer to this class of mechanisms as \emph{Vickrey mechanisms}. 

 Just as in a Vickrey auction, where only the second highest price is revealed\footnote{to ensure incentive-compatibility} and the highest price is kept private, the proposed Vickrey mechanisms generalize the noisy nearest neighbor selection by including the second nearest neighbor in the selection pool using a tuning parameter. 
The inclusion of the second nearest neighbor 
greatly reduces the empirical reconstruction risk on the original word. 
 
To select the tuning parameter above, we present a strategy based on optimizing the empirical privacy-utility tradeoff. The empirical privacy measurement is constructed in the context of analysis on de-identified text, which quantifies the risk on how well an adversary can reconstruct the original text based on the observed (possibly perturbed) text. The better the reconstruction, the lower the empirical privacy guarantee. This general framework allows comparing text generation mechanisms that use different distance metrics (see Section~\ref{sec:tuning}). 

We emphasize that our empirical privacy metric does not supersede the metric-DP guarantee; instead, it provides a new dimension along which different metric-DP mechanisms can be aligned. We say that, within the class of metric-DP mechanisms, an \emph{optimal} mechanism is the one that maximizes the empirical privacy guarantee while keeping the utility loss of the downstream task under some maximum tolerable budget. This definition for privacy-utility tradeoff, modeled as a constrained optimization problem, resembles the literature on protecting privacy for location data \cite{shokri2011quantifying,shokri2012protecting,clark2019privacy}. We extend the analysis for the broader class of metric-DP mechanisms. 
 Additionally, in our experiments, we demonstrate that our proposed Vickrey mechanisms outperform existing mechanisms with respect to the empirical privacy-utility tradeoff on real text classification datasets.
 


\paragraph{Related Work.} Metric-DP  \cite{andres2013geo,chatzikokolakis2013broadening,laud2020framework}, an extended notion of local DP \cite{kasiviswanathan2011can}, is a popular tool for privacy-preserving text analysis. A text generation mechanism that satisfies DP with respect to the hyperbolic distance metric was proposed in \cite{feyisetan2019leveraging}. This mechanism requires specialized training of word embeddings in the high-dimensional hyperbolic space. For word embeddings in the Euclidean space, like \textsc{GloVe} \cite{pennington2014glove} or \textsc{FastText} \cite{bojanowski2017enriching}, mechanisms like the Laplace mechanism ($L_2$ metric)~\cite{fernandes2019generalised,feyisetan2020privacy} and the Mahalanobis mechanism (using a regularized Mahalanobis metric) \cite{xu2020mahalanobis} have been proposed. However, a structured comparison of these different mechanisms remains unclear.

\emph{Empirical privacy measurements.} A variety of empirical techniques for privacy measurement have been proposed for many different applications. 
In the membership inference attack literature \cite{shokri2017membership,yeom2018privacy,salem2018ml,song2019auditing}, an AUC based detectability metric is commonly used to quantify the information leakage from machine learning models about their training data. However, the model trained on a given dataset can only serve as a proxy to estimate its privacy guarantee. Moreover, the detectability metric can vary across different machine learning models and implementations of the inference attack based auditors. 

Hypothesis testing based approaches have also been proposed to empirically estimate $\epsilon$ \cite{ding2018detecting,gilbert2018property,liu2019minimax}. However, the assumptions in these methods constrain their general applicability. 
In a recent line of work on privacy-preserving text analysis \cite{feyisetan2020privacy,xu2020mahalanobis}, privacy statistics defined as (i) probability of inputs not being redacted, and (ii) number of distinct outputs given a fixed input, have been used to characterize the empirical privacy of a text generation mechanisms. While those metrics are intuitive and descriptive, there is not a direct association that relates them to the privacy leakage. Within the class of metric-DP text generation mechanisms, the corresponding definition of empirical privacy-utility tradeoff is a constrained optimization to maximize the empirical privacy while keeping the utility loss under a preset budget.
This constrained setup can find its precedent in the location data privacy literature \cite{shokri2011quantifying,shokri2012protecting,clark2019privacy}. We differ in their approach as we require the optimal mechanism to also satisfy metric-DP. 


\section{The Class of Vickrey Mechanisms}

To motivate our construction of the Vickrey mechanisms, we begin by discussing the limitations of a general approach in the existing metric DP text generation mechanisms. We denote by $\mathcal{W} = \{w_1,\ldots,w_n\}$ the vocabulary set containing $n$ distinct words, and by $\phi: \mathcal{W}\to\mathbb{R}^p$ a fixed embedding function that maps each word in the vocabulary set to a $p-$dimensional real vector (referred to as the embedding for the word).  

A common first step is to sample an additive noise $Z$ from a density function $p(z)\propto\exp\{-d(z,0)\}$, where $d$ is the distance metric used in the mechanism\footnote{We use the standard definition of a \emph{metric}, which requires the distance function to satisfy (1) $d(x,x) = 0$ for all $x$; (2) $d(x,y) > 0$ for $y\neq x$; and, (3) the triangle inequality.}.
For example, the Laplace mechanism uses $d(x,y)=\|x-y\|_2$ (also known as Euclidean or $L_2$ distance), and the Mahalanobis mechanism uses $d(x,y)=\sqrt{(x-y)\Sigma^{-1}(x-y)}$ (also known as Mahalanobis distance), where $\Sigma$ is the sample covariance of the word embeddings.

Once the noise is sampled, it is then added to the input word embedding and the word with an embedding that is nearest to this noised embedding is chosen as the output: $$w_{output} = \arg\min_{w \in \mathcal{W}} d(\phi(w_{input})+Z, w).$$ 

A limitation of this noisy nearest neighbor selection is that when $|Z|$ is small (in particular, smaller than half the distance from the input word to its nearest neighbor), the first nearest neighbor to the noised embedding is the same as the original input word. The problem is exaggerated for rare words, which exist in the sparse regions of the embedding space and hence, do not get perturbed even for larger noise scales. This makes it easier for an adversary to reconstruct the original word, which may contain sensitive information (\emph{e.g.} street names).

The proposed Vickrey mechanisms generalize the noisy nearest neighbor selection step by distributing the selection probability between the first and second nearest neighbor\footnote{See Section~\ref{sec:general_construction} for a general construction using $k$ nearest neighbors and our experimental results for the same.} using a tuning parameter $t\in[0,1]$ (see Algorithm~\ref{alg_main}).
Intuitively, this generalization makes the reconstruction of the original input word harder (see Figures~\ref{fig1} and~\ref{combined}).

We capture our intuition for the claim above in Figure \ref{fig1}. For simplicity, the horizontal axis in both plots represents the one-dimensional embedding on a vocabulary containing only 5 words: (A, B, C, D, E). The vertical axis represents the output probability of each word through the mechanism. The plots represent the output probability in the mechanism for each of the 5 words, corresponding to the potential noised embedding values on the horizontal axis. 
The top plot represents the Laplace mechanism when only the first nearest neighbor to the noised embedding is feasible for selection $(t=0)$. In this case, all 5 curves are step functions since only the nearest neighbors are returned. 
The bottom plot shows the output probability for the Vickrey mechanisms, which always impart plausible deniability with another word when the noised embedding falls in any open interval. 

\begin{figure}[t]
    \centering
    \includegraphics[width=\columnwidth]{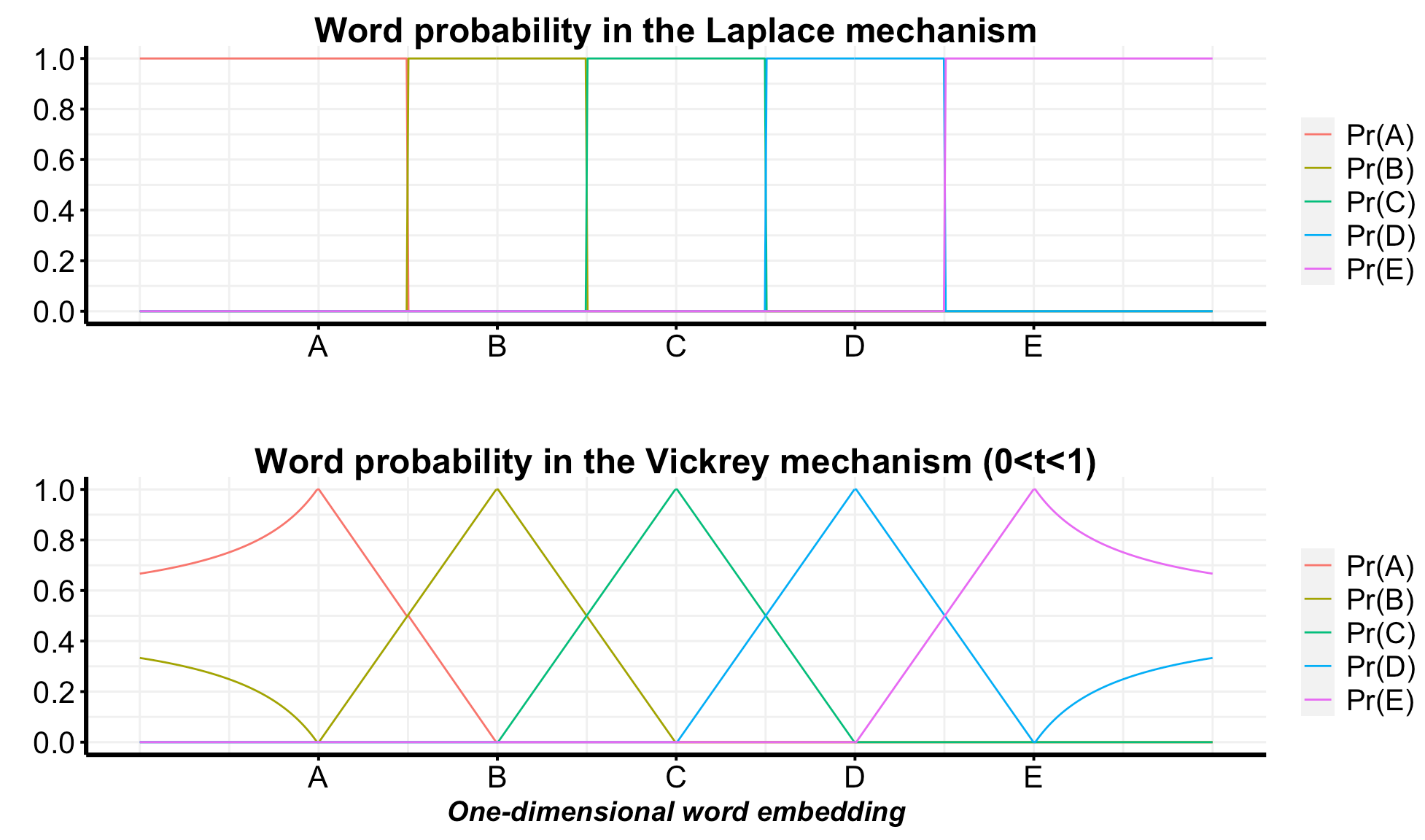}
    \caption{Word probability in the Laplace mechanism (top) and the Vickrey mechanism (bottom) at $0<t<1$ for each of the 5 words as a function of the noised one-dimensional embedding. The Vickrey mechanism in this example always has two candidate words as output.}
    \label{fig1}
\end{figure}
 
\begin{algorithm}[t]\small
\textbf{Input: }String $s=w_1w_2\ldots w_n$, metric $d$, privacy parameter $\epsilon$, tuning parameter $t\in[0,1]$ \\
 \SetAlgoLined
 \For{$w_i \in s$}{
 	Sample $Z$ with density $p(z)\propto \exp\{-\epsilon d(z,0)\}$. \\
 	Obtain $\hat{\phi}_i \gets \phi(w_i)+Z$.\\
 	Let $\tilde{w}_{i1} \gets \substack{\arg\min \\ w \in \mathcal{W}\setminus\{w_i\}} \|\hat{\phi}_i - \phi(w)\|_2$, and
 	$\tilde{w}_{i2} \gets \substack{\arg\min \\ w \in \mathcal{W}\setminus\{w_i, \tilde{w}_{i1}\}} \|\hat{\phi}_i - \phi(w)\|_2$.\\
 	Set $\hat{w}_i \gets \begin{cases} 
 		\tilde{w}_{i1} &\text{ with prob. }p(t,\hat{\phi}_i)\\
 		\tilde{w}_{i2} &\text{ with prob. }1 - p(t,\hat{\phi}_i)\\
 	\end{cases}$,where $p(t,\hat{\phi}_i) = \frac{(1-t)\|\phi(\tilde{w}_{i2})-\hat{\phi}_i\|_2}{t\|\phi(\tilde{w}_{i1})-\hat{\phi}_i\|_2+(1-t)\|\phi(\tilde{w}_{i2})-\hat{\phi}_i\|_2}$.\\
 	}
 \Return $\tilde{s}=\hat{w}_1\hat{w}_2\ldots \hat{w}_n$.
 \caption{The Vickrey Mechanism}
 \label{alg_main}
\end{algorithm}
\paragraph{Overview of Algorithm~\ref{alg_main}.} We outline the main steps for the class of Vickrey mechanisms in Algorithm \ref{alg_main}. For each word in the input, an additive noise $Z$ is sampled according to the  density function $p(z)\propto\exp\{-d(z,0)\}$. Then the Vickrey mechanism will select both the first and second nearest neighbor of the noised embedding as candidates, and randomly output one of them according to probabilities calibrated by their distances to the noised embedding using a tuning parameter $t$. The closer $t$ is to 1, the more the Vickrey mechanism favors the second nearest neighbor.

\noindent\textit{Privacy Analysis.} We formally prove that the Vickrey mechanism $M_t^\epsilon$ at privacy parameter $\epsilon>0$ enjoys $\epsilon$ metric-DP guarantee for any $t\in[0,1]$. 

\begin{theorem}\label{lemma1}
For any $t\in[0,1]$, $\epsilon>0$, metric $d$ and  $w,w',\hat{w}\in\mathcal{W}$, the Vickrey mechanism $M^{\epsilon}_t$ from Algorithm \ref{alg_main} satisfies metric-DP:
\begin{equation*}
\frac{\Pr\{M_t^\epsilon(w)=\hat{w}\}}{\Pr\{M^\epsilon_t(w')=\hat{w}\}}\leq \exp\paran{\epsilon d\{\phi(w),\phi(w')\}}.
\end{equation*} 
\end{theorem}

\begin{proof}
Define $Q^{w_j}_{w_i}=\{v\in\mathbb{R}^p: \|v-\phi(w_i)\|_2 < \|v-\phi(w_j)\|_2 < \min_{w\in\mathcal{W} \backslash \{w_i,w_j\}}\|v-\phi(w)\|_2\}$ to be the set that has $w_i$ and $w_j$ as the first and second nearest neighbors. 
Let $p_w(z)$ be the density function for the perturbed embedding conditional on the input $w$: $$p_w(z)\propto \exp\paran{-\epsilon d\{z-\phi(w), 0\}}$$
Since metric $d$ satisfies the triangle inequality, $$d\{z-\phi(w'),0\}-d\{z-\phi(w), 0\}\leq d\{\phi(w), \phi(w')\},$$ we obtain:
$$e^{-\epsilon d\{z-\phi(w), 0\}} \leq e^{\epsilon d\{\phi(w), \phi(w')\}}e^{-\epsilon d\{z-\phi(w'),0\}},$$
which is equivalent to the inequality $p_w(z)\leq e^{\epsilon d\{\phi(w), \phi(w')\}}p_{w'}(z)$.
For brevity, let 
\begin{align*}
    \alpha_{\hat{w}}^{w_j}(t,z)=&\frac{(1-t)\|z-\phi(w_j)\|_2}{t\|z-\phi(\hat{w})\|_2+(1-t)\|z-\phi(w_j)\|_2}
\end{align*} and $\rho(w,\hat{w}) = \Pr\{M^\epsilon_{t}(w)=\hat{w}\}$. Since $\rho(w,\hat{w})$ is a sum of partial probabilities in the areas where $\hat{w}$ is either the first or the second nearest neighbor to the noised embedding, we have:
\begin{align*}
\rho(w,\hat{w})
&= \sum_{j=1}\int_{Q_{\hat{w}}^{w_j}}p_w(z)\alpha_{\hat{w}}^{w_j}(t,z)dz \\ 
&\ \ \ \ \ + \sum_{i=1}\int_{Q^{\hat{w}}_{w_i}}p_w(z)\{1-\alpha_{w_i}^{\hat{w}}(t,z)\}dz \\
&\leq C(w,w') \biggl[\sum_{j=1}\int_{Q_{\hat{w}}^{w_j}}p_{w'}(z)\alpha_{\hat{w}}^{w_j}(t,z)dz\\
&\ \ \ \ \ + \sum_{i=1}\int_{Q^{\hat{w}}_{w_i}}p_{w'}(z)\{1-\alpha_{w_i}^{\hat{w}}(t,z)\}dz\biggr] \\
&=C(w,w') \Pr\{M^\epsilon_t(w')=\hat{w}\},
\end{align*}
where $C(w,w') = e^{\epsilon d\{\phi(w), \phi(w')\}}$, as desired.
\end{proof}


For our experiments, we use the Euclidean distance for $d$ so that the Vickrey mechanism reduces to the Laplace mechanism when $t=0$. In general, any distance function $d$ that satisfies the triangle inequality can be used to ensure the desired metric-DP guarantee (quantified by the parameter $\epsilon$).

\section{Tuning Parameter Selection}\label{sec:tuning}

We now discuss how to select the tuning parameter in Algorithm~\ref{alg_main}. We do this by optimizing an empirical formulation of the privacy-utility tradeoff. We discuss the details of this formulation next. 

\subsection{General Framework for Empirical Privacy Utility Tradeoff} 
Let  $M:\mathcal{W}\to\mathcal{W}$ denote some privacy-preserving text generation mechanism (that maps words to their noised versions).  
Define $f_M(w'|w)\triangleq\Pr\{M(w)=w'\}$ to be the probability of observing $w'$ as the output of the mechanism $M$ from the input word $w$. Note that this probability is conditioned on the knowledge of $w$. We assume a prior probability measure $\pi:\mathcal{W}\to[0,1]$, which represents the adversary's domain knowledge about the NLP task and distribution of words in the dataset under consideration. Depending on the use case, the prior distribution $\pi$ can be chosen as uniform, which means the user has no information on the word distribution in the context; or $\pi$ can be chosen as the empirical word distribution in the corpus on which the user wishes to perform text generation.

Given this formulation, we define the expected utility loss for mechanism $M$ as follows:
\begin{equation}\label{eq_utility}
    L_M\triangleq\sum_{w,w'\in\mathcal{W}}\pi(w)f_M(w'|w)d_L(w,w'),
\end{equation}
where $d_L:\mathcal{W}\times\mathcal{W}\to[0,\infty)$ is a utility-specific distance metric. The utility loss can be bounded as $L_M < C$ for some bound $C>0$, depending on the maximum tolerance for the underlying task.

To model the empirical privacy loss, we assume an informative adversary $\mathcal{A}$ that uses the prior $\pi$ and has full knowledge of the text generation mechanism $M$ and the parameter $\epsilon$ used (similar to~\cite{shokri2011quantifying,shokri2012protecting}). This adversary uses the posterior probability of each word given the observed perturbed output to make its inference: 
\begin{equation}\label{posterior}
    g_{\mathcal{A}}(\hat{w}|w')\triangleq \frac{\pi(w')f_M(\hat{w}|w')}{\sum_{w\in\mathcal{W}}\pi(w)f_M(w'|w)},
\end{equation}
Thus, from $\mathcal{A}$'s perspective, the expected inference error with respect to $M$ is given by:
\begin{equation}\label{eq_error}
    E_M=\sum_{w,w',\hat{w}}\pi(w)f_M(w'|w)g_{\mathcal{A}}(\hat{w}|w')d_E(\hat{w},w),
\end{equation}
where $d_E:\mathcal{W}\times\mathcal{W}\to[0,\infty)$ is some privacy-specific distance metric. Our goal is, therefore, to find a mechanism within the class of metric-DP mechanisms $\mathcal{M}$ that maximizes the expected inference error $E_M$ while keeping the utility loss $L_M$ below $C$:
\begin{equation}\label{eq_maximize}
M_{\text{optimal}} = \arg\max_{M\in\mathcal{M}}E_{M},\;\;s.t.\;\;L_{M} < C.
\end{equation}
To compare different mechanisms, we will compare their expected inference error $E_M$ under different tolerance thresholds on the expected utility loss $L_M$. We favor mechanisms with high $E_M$, while maintaining $L_M < C$.
%

Note that $d_L$ and $d_E$ do not have to be the same distance metrics. For instance, $d_L$ can depend on the downstream machine learning tasks, like the absolute difference in classification error, perplexity or even cross-entropy loss. From the privacy perspective, a natural choice is $d_E(w,\hat{w})=\mathbbm{1}(\hat{w}\ne w)$, which means the adversary attempts to retrieve the original word from the redacted output and considers the inference attack successful if the inferred word is the exactly same as the input word. Based on applications, the adversary can also choose $d_E$ to be the Euclidean distance such that the goal of the inference attack is to have the inferred word as close to the original word as possible.

\begin{algorithm}[t]\small
\textbf{Input: }Vocabulary $\mathcal{W}$, maximum utility loss $C$, sampler for the Vickrey mechanism $M_t^{\epsilon}$ at any privacy parameter $\epsilon>0$ and tuning parameter $t\in[0,1]$ \\
 \SetAlgoLined
 Initialize $E_{\max} \gets 0, \epsilon \gets \epsilon_0, t\gets 0$  \\
 \While{$L_{M_t^{\epsilon}}\geq C$}{
set $\epsilon=2\epsilon$ 
}
 set $E_{\max}\gets E_{M^{\epsilon}_t},\epsilon_{opt}\gets \epsilon$,
     $t_{opt} \gets 0$\\
     \For{$t\in[0.05, 0.1, \ldots,1]$}{
       If $L_{M^{\epsilon}_{t}}\leq C$ and $E_{M^{\epsilon}_t}>E_{\max}$, \\
     \hspace{0.1in} set $E_{\max}\gets E_{M^{\epsilon}_{t}}$,  $\epsilon_{opt}\gets\epsilon$,
     $t_{opt}\gets t$\\
     }
 \Return $\epsilon_{opt},t_{opt}$.
 \caption{Empirical Parameter Selection for the Vickrey Mechanism}
 \label{alg_tune}
\end{algorithm}

\subsection{Selecting the Tuning Parameter} 
We outline the main steps for optimizing the privacy parameter $\epsilon$ as well as the tuning parameter $t$ in Algorithm \ref{alg_tune}. This optimization is with respect to the empirical privacy-utility tradeoff as laid out in (\ref{eq_maximize}). We initialize with the privacy parameter $\epsilon=\epsilon_0$ at some small initial value $\epsilon_0$ and tuning parameter $t=0$, so that the initial mechanism is essentially a metric-DP mechanism that implements the noisy first nearest neighbor selection. Next, we incrementally double the value of $\epsilon$ until the expected utility loss $L_{M^\epsilon_t} < C$ (recall that a smaller $\epsilon$ typically has larger utility loss\footnote{An implicit assumption we make in Algorithm~\ref{alg_tune} is that $L_M$ increases monotonically with $\epsilon$, following the intuition that a larger noise scale leads to larger utility loss. We defer the discussion around relaxing this assumption to future work.}). Once the maximum $\epsilon$ is obtained, we iterate over different values of $t$ between 0 and 1 (since a monotonicity assumption cannot be made here in general for the behavior of $E_M$). The final parameters $\epsilon_{opt}$ and $t_{opt}$ chosen provide the highest empirical privacy while keeping the utility loss within the specified budget. More importantly, Theorem~\ref{lemma1} ensures that the selected mechanism enjoys at least as much metric DP as the initial mechanism, which implements only the nearest neighbor selection.


\section{Experimental Results}\label{sec:experiments}

\paragraph{Setup.} We evaluate the performance of the proposed Vickrey mechanisms in terms of the empirical privacy-utility tradeoff on three datasets:

\begin{itemize}
    \item The \emph{Product Reviews dataset} consists of a list of 2,006 positive sentiment words and 4,783 negative sentiment words extracted from customer reviews \cite{hu2004mining}. This is a word-level dataset and the metric $d_L$ in expected utility loss is $\mathbbm{1}\{\textrm{sentiment}(w')\ne\textrm{sentiment}(w)\}$, i.e., the loss is incremented when a positive sentiment word is redacted into a negative sentiment word, or vice versa.
    \item The \emph{IMDb Movie Reviews dataset} \cite{maas2011learning} has a total vocabulary size of 145,901, where a pre-specified set of 26,078 words are subject to redaction in the text generation mechanism (those are the words selected for adversarial model training in \cite{jia2019certified}). The utility task is the sentence-level binary sentiment classification, where the underlying model is a bidirectional LSTM using 90\% of the data for training and 10\% for testing.
    \item The \emph{Twitter dataset}  contains 7,613 tweets, with a vocabulary of 22,013 words\footnote{\url{https://www.kaggle.com/c/nlp-getting-started}}. Each tweet is associated with a label indicating whether the tweet describes a disaster event or not. The classification model is a bidirectional LSTM using 9:1 data split for training/testing.
\end{itemize}
 
For all three datasets, we consider both 300-dimensional \textsc{GloVe} embeddings \cite{pennington2014glove} and 300-d \textsc{FastText} embeddings \cite{bojanowski2017enriching}. The empirical privacy measurement uses the adversary's expected inference error rate, i.e. $d_E(\hat{w},w)=\mathbbm{1}(\hat{w}\ne w)$. 
The utility-specific metric $d_L$ is chosen to be the misclassification error rate. 
The prior word distribution is chosen to be the empirical word distribution in the dataset, because we want to assume an informative adversary so as not to underestimate the privacy risk. In the Vickrey mechanism, the distance function is the Euclidean distance, so that $t=0$ is equivalent to the Laplace mechanism \cite{feyisetan2020privacy}. We also compare our results with the Mahalanobis mechanism \cite{xu2020mahalanobis}.

\begin{figure*}[t]
    \centering
    \includegraphics[width=0.95\textwidth]{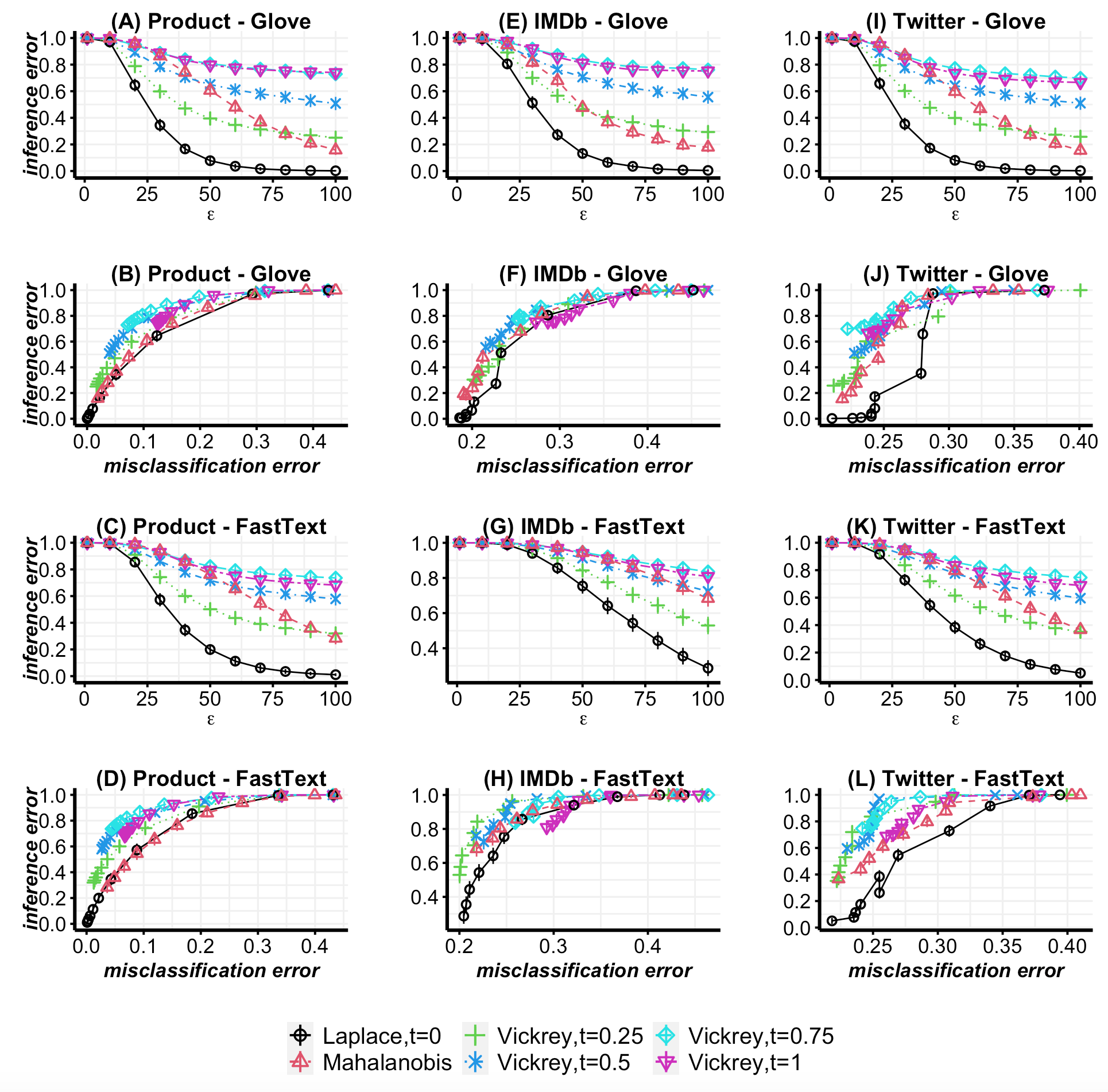}
    \caption{(A): empirical privacy vs $\epsilon$ on Product Reviews using 300-d \textsc{GloVe}. (B): empirical privacy vs utility loss on Product Reviews using 300-d \textsc{GloVe}. (C): empirical privacy vs $\epsilon$ on Product Reviews using 300-d \textsc{FastText}. (D): empirical privacy vs utility loss on Product Reviews using 300-d \textsc{FastText}. (E) - (H) are for IMDb reviews, and (I) - (L) are for Twitter dataset.  }
    \label{combined}
\end{figure*}

\paragraph{Results and Observations.} In Figure 2(A) - 2(D) shows the empirical privacy-utility tradeoff on the Product Reviews between the Laplace mechanism, Mahalanobis mechanism, and the Vickrey mechanisms with tuning parameter at to 0.25, 0.5, 0.75, and 1. The vertical axis in all plots represents the adversary’s inference error in the mechanism. The error bars are computed over 100 runs. 
In the 2(A), the horizontal axis is the privacy budget $\epsilon$. When $\epsilon$ approaches 0, the inference error in all mechanisms approach 1, which is expected because magnitude of the additive noise is large. 
When $\epsilon$ increases, the inference error drops, but the drop in Laplace mechanism is much faster than the other mechanisms. 
It is worth noticing that the curves for Laplace mechanism and the Vickrey mechanisms are mostly parallel with each other: when $t$ increases from 0 to 0.75, a higher value of $t$ is better in terms of empirical privacy at the same $\epsilon$; but when $t$ increases to 1, the empirical privacy will not further increase since the randomness in noisy selection between the first and second nearest neighbor is replaced by the deterministic selection of the second nearest neighbor, which makes the adversary's inference attack easier by finding the second nearest neighbor. 
However, Vickrey mechanism at $t=1$ still dominates Laplace mechanism which only selects the noisy first nearest neighbor for redaction.
The slope for the Mahalanobis mechanism is different from the rest, where intersects with Vickrey mechanisms with different $t$ at different $\epsilon$.
At $\epsilon=100$, the baseline Laplace mechanism has negligible inference error, which means the adversary can almost always make correct guesses, whereas in the other mechanisms the error is still substantial. 

Figure 2(B) plots inference error vs. misclassification error of words (positive words to negative words, or vice versa). When vertically slicing the plot, we see that for each utility loss budget greater than 0.1, a larger value of $t$ will result in a better privacy guarantee.   
When capped at a maximum $\epsilon=100$, the curves with a higher $t$ value will have a higher minimum feasible misclassification error, which is around around 0.02 for $t=0.25$ and Mahalanobis, about 0.04 for $t=0.5$, about 0.06 for $t=0.75$, and about 0.1 for $t=1$. 
This is expected because as more weight is put on the second nearest neighbor, the utility loss becomes larger at large $\epsilon$ values (small noise), because it is more likely that the original word will get changed to its neighbors.
But this loss is upper bounded by the nearest neighbor replacement, which tends to be small as is shown in the experiments in the paper.
The plot suggests that if the user has a maximum utility loss budget of 0.06, they should go with the Vickrey mechanism at $t=0.75$ because when slicing vertically at misclassification error of 0.06, the green curve for $t=0.75$ attains a higher empirical privacy than the other mechanisms. However, when the utility loss budget is 0.02, the user should choose $t=0.25$ because the green line is on top of the other curves (red and black) that can achieve the utility loss of 0.02 (the blue, cyan, and purple curves cannot achieve utility loss within 0.02 when $\epsilon$ is capped at 100). 
Figure 2(C) and 2(D) on Product Reviews using 300-d \textsc{FastText} embedding show similar patterns as those in 2(A) and 2(B) in terms of the privacy-utility tradeoff.

The results and interpretations are qualitative similar in Figure 2(E) - 2(H) on IMDb Movie Reviews and in Figure 2(I) - 2(L) on Twitter. 
In empirical privacy vs $\epsilon$ plots, the  Laplace mechanism consistently has a lower value of empirical privacy measure than the Vickrey mechanism and the Mahalanobis mechanism. This gap in adversary's inference error becomes wider as $\epsilon$ increases. In the privacy vs utility loss plots, the difference between mechanisms is more significant on Twitter than  on IMDb reviews. 
The patterns are consistent across plots, which both show that the Vickrey mechanism can improve the privacy-utility tradeoff beyond the baseline mechanism.  

The difference in the result between \textsc{GloVe} and  \textsc{FastText}, particularly in 2(E) vs. 2(G) and 2(I) vs. 2(K), is due to the difference in inter-word distance distributions between the two embedding spaces (see Figure 1 in \cite{feyisetan2020privacy}). In particular, the inter-word distances are generally smaller in \textsc{FastText} than in \textsc{GloVe}, so that for a fixed noise scale $\epsilon$, the inference error is expected to be larger in \textsc{FastText} than in \textsc{GloVe}.

\section{Generalizing Vickrey Mechanism Beyond the Second Nearest Neighbor}\label{sec:general_construction}

\begin{figure}[t]
    \centering
    \includegraphics[width=\columnwidth]{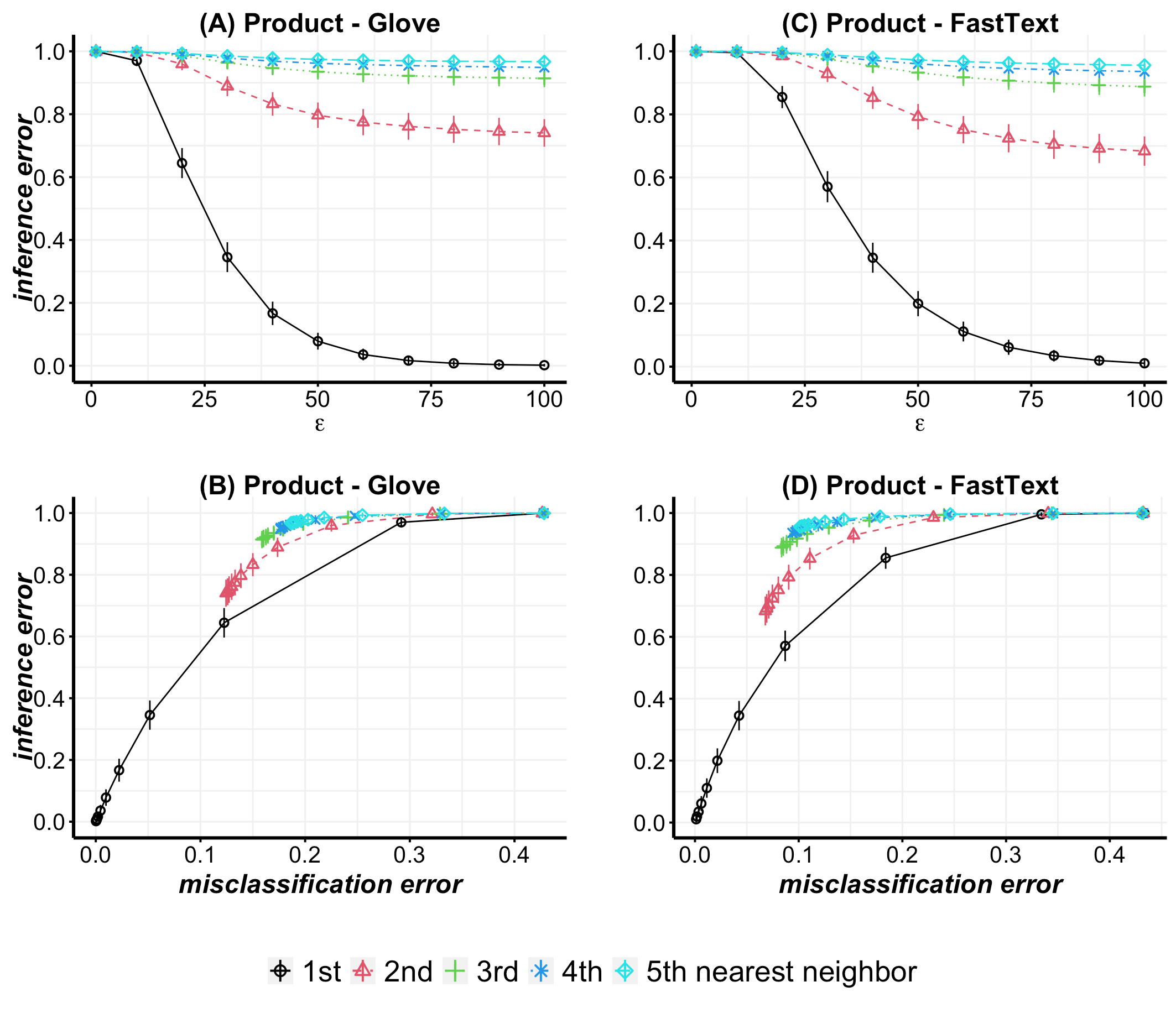}
    \caption{(A): empirical privacy vs $\epsilon$ on Product Reviews using 300-d \textsc{GloVe}. (B): empirical privacy vs utility loss on Product Reviews using 300-d \textsc{GloVe}. (C): empirical privacy vs $\epsilon$ on Product Reviews using 300-d \textsc{FastText}. (D): empirical privacy vs utility loss on Product Reviews using 300-d \textsc{FastText}.}
    \label{knn}
\end{figure}

By a random selection of both the first and the second nearest neighbor to the noised embedding, we have shown that the Vickrey mechanism can empirically improve the privacy-utility tradeoff upon the existing Laplace and Mahalanobis mechanisms. A natural generalization is to extend the selection to $k \geq 2$ nearest neighbors (see Algorithm \ref{alg_generalized}). 

Algorithm \ref{alg_generalized} presents the outline of the generalized Vickrey mechanism that randomly chooses among the noisy $k$ nearest neighbors as output, where the selection probability is inversely associated with their distance to the noised embedding. Similar to Algorithm~\ref{alg_tune}, the tuning parameters are selected as to optimize for the empirical privacy-utility tradeoff, but the selection process will be more challenging because the optimization space is unbounded. We defer the details of this optimization to future work. However, we formally state in Theorem \ref{lemma2} the metric-DP guarantee from the generalized Vickrey mechanism in Algorithm \ref{alg_tune}. Due to space constraints, we defer the details of the proof since it is similar to that for Theorem \ref{lemma1}.

\begin{theorem}\label{lemma2}
For any $t=[t_1,\ldots,t_k]\in [0,\infty)^k$, $k\in\mathbb{Z}_+$, $\epsilon>0$ and  $w,w',\hat{w}\in\mathcal{W}$, the generalized Vickrey mechanism $M^{\epsilon}_{t}$ from Algorithm \ref{alg_generalized} satisfies $\epsilon$ metric-DP for any metric $d$.
\end{theorem}

In Figure \ref{knn} , we compare 5 generalizations of the Vickrey mechanism that deterministically select the noisy $1^{st},\dots,5^{th}$ neighbors as the output on the Product Reviews data using both 300-d \textsc{GloVe} and  300-d \textsc{FastText}.
We can see that the improvement is most significant between the $1^{st}$ and $2^{nd}$ nearest neighbor.
It also shows that there is benefit in introducing the $3^{rd}$ nearest neighbor into the selection pool, while no big difference is found beyond the $3^{rd}$ neighbor. 
 
\begin{algorithm}[t]\small
\textbf{Input: }String $s=w_1w_2\ldots w_n$, metric $d$, privacy parameter $\epsilon$, tuning parameters $t_1,\ldots,t_k > 0$ \\
 \SetAlgoLined
 \For{$w_i \in s$}{
 	Sample $Z$ with density $p(z)\propto \exp\{-\epsilon d(z,0)\}$ \\
 	Obtain $\hat{\phi}_i \gets \phi(w_i)+Z$\\
 	Let $\tilde{w}_{i1} \gets \substack{\arg\min \\ w \in \mathcal{W}\setminus\{w_i\}} \|\hat{\phi}_i - \phi(w)\|_2$  \\
 	\nonl\hspace{0.2in}	\ldots \\
 	\nonl$\ \ \ \ \ \ \tilde{w}_{ik} \gets \substack{\arg\min \\ w \in \mathcal{W}\setminus\{w_i, \tilde{w}_{i1},\ldots,\tilde{w}_{ik}\}} \|\hat{\phi}_i - \phi(w)\|_2$.\\
 	Set $\hat{w}_i \gets \tilde{w}_{ir}$ with prob. $p_r(t_{1,\dots,k},\hat{\phi}_i)$, where $p_r(t_{1,\dots,k},\hat{\phi}_i) = \frac{\exp\{- t_r\|\phi(\tilde{w}_{ir}-\hat{\phi}_i\|_2\}}{\sum_{j} \exp\{-t_j\|\phi(\tilde{w}_{ij}-\hat{\phi}_i\|_2\}}$ for all $r \in [k]$.
 	}
 \Return $\tilde{s}=\hat{w}_1\hat{w}_2\ldots \hat{w}_n$.
 \caption{\small Generalized Vickrey Mechanism}
 \label{alg_generalized}
\end{algorithm}

\section{Discussion and Conclusion}

In this paper, we present a measurement framework to quantify the empirical privacy-utility tradeoff for metric-DP text generation mechanisms, where the empirical privacy metric is the reconstruction risk of the original text based on the redacted text. 
We adopt a constrained optimization setup, where within the class of metric-DP mechanisms, we maximize the empirical privacy guarantee while keeping the machine learning utility loss under a pre-specified tolerance.
A novel class of Vickrey mechanism is proposed, which not only enjoys metric-DP but also optimizes the privacy-utility tradeoff within the constraint.
We apply our methodology to the three text classification datasets and demonstrate how to empirically compare the privacy-utility tradeoff as well as how to choose the optimal parameter setting according to the constrained optimization. Our results show superior performance when compared to existing mechanisms.

Our analysis in this paper leaves ample room for further investigation. An ongoing work we are exploring is the inclusion of contextual information into the probability calibration between the two nearest neighbors. We leave it as an interesting open problem to explore how the choice of $k^{th}$ neighbor impacts the tradeoff in this scenario, since contextual signals will likely restrict the set of candidate words we can choose from.

\bibliographystyle{acl_natbib}

\bibliography{mybib}


\end{document}